\documentclass[11pt]{article}
\usepackage{amsmath, amssymb, amscd, amsthm, amsfonts}
\usepackage{graphicx}
\usepackage{hyperref}
  
\usepackage[square,sort,comma,numbers]{natbib}
\usepackage[utf8]{inputenc} 
\usepackage[T1]{fontenc}    
\usepackage{hyperref}       
\usepackage{url}            
\usepackage{booktabs}       
\usepackage{amsfonts}       
\usepackage{nicefrac}       
\usepackage{microtype}      
\usepackage[dvipsnames]{xcolor}

\newtheorem{definition}{Definition}

\DeclareMathOperator*{\argmin}{arg\,min}

\newtheorem{theorem}{Theorem}
\usepackage{graphicx,float} 
\usepackage{amssymb}
\usepackage{algorithm}
\usepackage{algorithmic}
\usepackage{amsmath,bm}
\usepackage{makecell}

\usepackage{adjustbox}
\usepackage{balance}
\usepackage{mathtools}
\usepackage{multirow}

\usepackage{footnote}
\usepackage{etoolbox}
\makeatletter
\patchcmd{\@makecaption}
  {\scshape}
  {}
  {}
  {}
\makeatother
\usepackage{xcolor}

\usepackage{graphicx,float} 
\usepackage{makecell}

\usepackage{booktabs}
\usepackage{lipsum}
\usepackage{multicol}
\usepackage[capitalize,noabbrev]{cleveref}

\theoremstyle{plain}
\theoremstyle{definition}
\theoremstyle{remark}

\newcommand\blfootnote[1]{%
  \begingroup
  \renewcommand\thefootnote{}\footnote{#1}%
  \addtocounter{footnote}{-1}%
  \endgroup
}
\makeatletter

\title{A principled approach to model validation in domain generalization}

\author{Boyang Lyu$^{1\dagger }$, Thuan Nguyen$^{1,2\dagger }$, Matthias Scheutz$^2$, Prakash Ishwar$^3$, Shuchin Aeron$^1$
  \thanks{Author affiliations  \newline $^1$ - Tufts University, Dept. of ECE, $^2$ - Tufts University, Dept. of CS., $^3$ - Boston University, Dept. of ECE  \newline $^\dagger$ - These authors contributed equally.
 \textbf{Corresponding authors}: Boyang Lyu, email: \url{Boyang.Lyu@tufts.edu}.} 
  } 
\date{}
\begin{document}
\maketitle

\begin{abstract}
Domain generalization aims to learn a model with good generalization ability, that is, the learned model should not only perform well on several seen domains but also {on} unseen domains with different data distributions.
{State-of-the-art domain generalization methods typically train a representation function followed by a classifier jointly to minimize both the classification risk and the domain discrepancy. However, when it comes to model selection, most of these methods rely on traditional validation routines that select models solely based on the lowest classification risk on the validation set.} In this paper, we theoretically demonstrate a trade-off between minimizing classification risk and mitigating domain discrepancy, \textit{i.e.,} it is impossible to achieve the minimum of these two objectives simultaneously. Motivated by this theoretical result, we propose a novel model selection method suggesting that the validation process {should} account for both the classification risk and the domain discrepancy. {We validate the effectiveness of the proposed method by numerical results on several domain generalization datasets.}
\end{abstract}
\section{Introduction and related work}
\blfootnote{© 2023 IEEE. Personal use of this material is permitted. Permission from IEEE must be obtained for all other uses, in any current or future media, including reprinting/republishing this material for advertising or promotional purposes, creating new collective works, for resale or redistribution to servers or lists, or reuse of any copyrighted component of this work in other works.}
The success of traditional machine learning methods relies on an important assumption that the training and the test data are independent and identically distributed (\textit{i.i.d}). However, in many real-world scenarios, the distributions of data in the training set and test set are not identical due to the ``distribution-shift" phenomenon. Mitigating the problem caused by the distribution shift is the primary goal of the Domain Generalization (DG) problem, where a model is trained using data from several seen domains but later will be applied to unseen (unknown but related) {domains} with different data distributions.

To address DG problem, a large number of methods {consider} training a representation function that can learn domain-invariant features by minimizing the domain discrepancy in the representation space \cite{arjovsky2020invariant,lyu2021barycentric,nguyen2022conditional,li2021invariant,sun2016deep,li2018domain}. Though the domain discrepancy has been accounted for at the training step, few works considered it for model selection at the validation step \cite{9847099}. Indeed, following traditional machine learning settings, most of the state-of-the-art DG methods form a validation set using a small portion of data from all seen domains and select the model that achieves the lowest classification risk or highest classification accuracy on it. However, unlike the traditional machine learning settings where a model with {lower} classification risk on the validation set is likely to perform {better} on the test set, we theoretically show that for DG problem, where the \textit{i.i.d} assumption does not hold, selecting the model with minimum classification risk may enlarge the domain discrepancy, subsequently leading to a non-optimal model on the unseen domain. We thus argue that one needs to consider both the classification risk and the domain discrepancy for selecting good models on unseen domains.  

We summarize our contributions as follows:

\begin{enumerate}
    \item We theoretically show that there is a trade-off between minimizing classification risk and domain discrepancy. This trade-off leads to the conclusion that {only targeting a model with the lowest classification risk on the validation set may encourage distribution mismatch between domains}  (enlarging domain discrepancy), and {reduce the model's generalization ability.}
  
    \item {Based on our theoretical result and considering the limited attention given to DG-specific validation processes, we propose a simple yet effective validation/model selection method that integrates both the classification risk and domain discrepancy as the validation criterion. We further demonstrate the effectiveness of this approach on various DG benchmark datasets.}
\end{enumerate}

The trade-off between minimizing the classification risk and domain discrepancy has been mentioned in the literature \cite{ben2006analysis,zhao2019learning}\footnote{The works in \cite{ben2006analysis,zhao2019learning} are for domain adaptation, not domain generalization. However, one may derive a similar conclusion by replacing the ``source domain" with seen domain and the ``target domain" with unseen domain.}. Shai \textit{et al.}  \cite{ben2006analysis} constructed an upper bound on the risk of the target domain, composed of the risk from the source domain and the discrepancy between the target and source domains. The authors suggested that there must be a trade-off between minimizing the domain discrepancy and minimizing the seen domain's risk {but} did not propose any further details on how this trade-off is determined and characterized. Zhao \textit{et al.} \cite{zhao2019learning} showed that the sum of the risks from source and target domains is lower bounded by the distribution discrepancy between domains. If the discrepancy between domains is large, one can not simultaneously achieve small risks {on} both domains. Though sharing some similarities, our theoretical result differs from \cite{zhao2019learning} since Zhao \textit{et al.} considered the trade-off between minimizing the risks of different domains rather than the trade-off between optimizing the classification risk and the domain discrepancy. On the other hand, {most DG works adopt the model selection methods following the traditional machine learning settings,} \textit{i.e.}, a validation set is first formed by combining small portions of data from all seen domains and the model that produces the lowest classification risk or highest classification accuracy on the validation set is then selected. 

To the best of our knowledge, there are only a few works that explore {new} model selection methods {under} DG {settings} \cite{wald2021calibration,ye2021towards,sagawa2019distributionally,albuquerque2019generalizing,arpit2021ensemble}. The most related work of this study is \cite{albuquerque2019generalizing}, where the authors mentioned that they use the training loss (including both classification risk and adversarial domain discrepancy loss) on the validation set for model selection. However, it is not clear from their paper {and} their released code how the classification risk and the adversarial domain discrepancy loss are used to validate the model and how these two terms are balanced. 
In contrast, we propose an alternative approach for combining the classification risk and the domain discrepancy loss in a meaningful way in light of our theoretical results.

\section{Problem Formulation}
\label{sec: problem setup}

\subsection{Notations}
{Let $\mathcal{X}$, $\mathcal{Z}$, $\mathcal{Y}$ denote the input space, the representation space, and the label space, $\mathcal{D}^{(s)}$ and $\mathcal{D}^{(u)}$ represent the seen and unseen domain, respectively. $f: \mathcal{X} \rightarrow \mathcal{Z}$ and $g: \mathcal{Z} \rightarrow \mathcal{Y}$ are the representation function and the classifier. We use capital letters for the random variables in different spaces and lowercase letters for samples. Specifically, we denote $X$ as the input random variable, $Z$ as the extracted feature random variable, and $Y$ as the label random variable. 
The input samples, feature samples, and  labels of input samples are denoted as $\bm{x}, \bm{z}$, and $y(\bm{x})$, respectively. {Finally, we use $p^{(s)}(\cdot)$ and $p^{(u)}(\cdot)$ to denote the distributions or joint distributions corresponding to the variables inside the bracket on seen domain and unseen domain, respectively.}}

\subsection{Problem formulation}

For a representation function $f$ and a classifier $g$, the classification risk induced by $f$ and $g$ on seen domain is{:}
\begin{eqnarray}
    C^{(s)}\!(f,g) \!&=&\! \int_{\bm{x} \in \mathcal{X}} p^{(s)}(\bm{x}) \ell(g(f(\bm{x})), y^{(s)} (\bm{x}) ) d \bm{x} \nonumber\\
    &=&\! \int_{\bm{x} \in \mathcal{X}} \! \int_{\bm{z} \in \mathcal{Z}} \! p^{(s)} \! (\bm{x},\bm{z}) \ell(g(\bm{z}), y^{(s)} (\bm{x})) d \bm{x} d \bm{z}
\end{eqnarray}where $\ell(\cdot,\cdot)$ is a distance measure {that} quantifies the mismatch between the label outputted by classifier $g$ and the true label.

For a representation function $f$, the distribution discrepancy between seen and unseen domains induced by $f$ is{:}
\begin{equation}
\label{eq: epsilon}
D(f) = d(p^{(u)}(Y,Z) || p^{(s)}(Y,Z))
\end{equation}where $d(\cdot || \cdot)$ is a divergence measure between two distributions. Indeed, to deal with the ``distribution-shift", one usually looks for a mapping $f$ such that the discrepancy between distributions of seen and unseen domains $D(f)$ is small \cite{nguyen2022joint,david2010impossibility}. 

{A large number of DG works focus on training a model that minimizes both the classification risk $C^{(s)}(f,g)$ and the discrepancy $D(f)$ using data from seen domains \cite{arjovsky2020invariant,lyu2021barycentric,nguyen2022conditional,li2021invariant,sun2016deep,li2018domain}}. Note that while $C^{(s)}(f,g)$ can be directly minimized, one usually {need to} approximately/heuristically optimize $D(f)$ by optimizing the distribution discrepancy between several seen domains. {Since there are already well-established theoretical and empirical works on minimizing the classification risk and domain discrepancy, our {work} aims to highlight the trade-off between these two objectives (Sec~\ref{sec: trade-off}) and argues that taking both objectives into account during model selection can improve model's performance on unseen domains (Sec.~\ref{sec: method}).}s  

\section{Trade-off between classification risk and domain discrepancy} 
\label{sec: trade-off}

We first begin with a definition.
\begin{definition}[Classification risk-domain discrepancy function]
\label{def: 3}
For any representation function $f$ and classifier $g$, define:
\begin{equation}
\begin{aligned}
&T(\Delta) = \min_{f: \mathcal{X} \rightarrow \mathcal{Z}} D(f) = \min_{f: \mathcal{X} \rightarrow \mathcal{Z}} d(p^{(u)}(Y,Z) || p^{(s)}(Y,Z))\\
& \textrm{s.t.} \quad C^{(s)}(f,g) \!=\!\int_{\bm{x} \in \mathcal{X}} \! p^{(s)}(\bm{x})  \ell(g(f(\bm{x})), y^{(s)}(\bm{x}) ) d \bm{x} \leq \Delta \label{eq: classification risk-domain alignment function}
\end{aligned}
\end{equation}where $\Delta$ is a positive number, $\ell(\cdot, \cdot)$ is a distance measure, and $d(\cdot || \cdot)$ is a divergence measure.
\end{definition}

$T(\Delta)$ is the minimal discrepancy between the joint distribution of the unseen domain and seen domain if the classification risk on seen domain $C^{(s)}(f,g)$ does not exceed a positive threshold $\Delta$. Next, we formally show that there is a trade-off between minimizing the distribution discrepancy $D(f)$ and minimizing the classification risk $C^{(s)}(f,g)$. 

\begin{theorem}[Main result]
\label{theorem: 1}
If the divergence measure $d(a || b)$ is convex (in both $a$ and $b$), for a fixed classifier $g$, $T(\Delta)$ defined in (\ref{eq: classification risk-domain alignment function}) is monotonically non-increasing, and convex.
\end{theorem}

\begin{proof}
The proof of this theorem is mainly based on the proposed approach in Rate-Distortion theory \cite{cover1999elements}. Particularly, consider two positive numbers $\Delta_1$ and $\Delta_2$, and assume $\Delta_1 \leq \Delta_2$. For a given classifier $g$, we use $\mathcal{F}_{\Delta_1}$ and $\mathcal{F}_{\Delta_2}$ to denote the sets of mappings $f$ such that $C^{(s)}(f,g) \leq \Delta_1$ and $C^{(s)}(f,g) \leq \Delta_2$, respectively. First, we show that $T(\Delta)$ is non-increasing. Indeed, from $\Delta_1 \leq \Delta_2$, $\mathcal{F}_{\Delta_1} \subset \mathcal{F}_{\Delta_2}$:
\begin{eqnarray}
T(\Delta_1) &=& \min_{f \in \mathcal{F}_{\Delta_1}} d(p^{(u)}(Y,Z) || p^{(s)}(Y,Z)) \nonumber\\
&\geq& \min_{f \in \mathcal{F}_{\Delta_2}} d(p^{(u)}(Y,Z) || p^{(s)}(Y,Z)) = T(\Delta_2). \nonumber
\end{eqnarray}Second, to prove the convexity of  $T(\Delta)$, we show that:
\begin{equation}
\label{eq: prove convexity of T(V)}
 \lambda T(\Delta_1) \!+\! (1\!-\!\lambda)T(\Delta_2) \!\geq\! T(\lambda \Delta_1 \!+\! (1\!-\!\lambda) \Delta_2), \forall  \lambda \! \in \! [0,1]. \end{equation}To prove (\ref{eq: prove convexity of T(V)}), we need some additional notations. Define:
\begin{equation}
\begin{aligned}
f_1 \quad &= \argmin_{f: \mathcal{X} \rightarrow \mathcal{Z}} D(f) \quad
\textrm{s.t.} \quad & C^{(s)}(f,g) \leq \Delta_1, \label{eq: tata20}
\end{aligned}
\end{equation}

\begin{equation}
\begin{aligned}
f_2 \quad &= \argmin_{f: \mathcal{X} \rightarrow \mathcal{Z}} D(f) \quad
\textrm{s.t.} \quad & C^{(s)}(f,g) \leq \Delta_2. \label{eq: tata21}
\end{aligned}
\end{equation}Note that for any $f$, $Y \rightarrow X \rightarrow Z$ forms a Markov chain, thus:
\begin{equation}
\label{eq: markov 1}
 p^{(u)}(Y,Z) = p^{(u)}(Y|X) \, p^{(u)}(X,Z),   
\end{equation}
\begin{equation}
\label{eq: markov 2}
 p^{(s)}(Y,Z) = p^{(s)}(Y|X) \, p^{(s)}(X,Z),   
\end{equation}where $p^{(u)}(Y|X)$ and $p^{(s)}(Y|X)$ are independent of $f$ and only depend on the conditional distributions of label and data on seen and unseen domains. 

{Let $p^{(u)}_1(Y,Z)$, $p^{(s)}_1(Y,Z)$ be the joint distributions of $Y$ and $Z$ on unseen and seen domain produced by $f_1$, and similarly $p^{(u)}_2(X,Z)$, $p^{(s)}_2(X,Z)$ be the joint distributions produced by $f_2$.} Define:
\begin{eqnarray}
\label{eq: tata29}
p^{(u)}_{\lambda}(X,Z) =\lambda p^{(u)}_1(X,Z) + (1-\lambda) p^{(u)}_2(X,Z),
\end{eqnarray}
\begin{eqnarray}
\label{eq: tata292}
p^{(s)}_{\lambda}(X,Z) =\lambda p^{(s)}_1(X,Z) + (1-\lambda) p^{(s)}_2(X,Z).
\end{eqnarray}
By definition, the left hand side of (\ref{eq: prove convexity of T(V)}) can be rewritten by:
\begin{eqnarray}
&& \! \lambda T(\Delta_1) + (1-\lambda)T(\Delta_2) \nonumber\\
&=& \!\lambda d(p^{(u)}_1(Y,Z) \, || \, p^{(s)}_1(Y,Z) ) \nonumber\\
&+&\! (1-\lambda) d(p^{(u)}_2(Y,Z) \, || \, p^{(s)}_2(Y,Z)) \nonumber \\
&=&\! \lambda d(p^{(u)}(Y|X) p^{(u)}_1(X,Z)  || p^{(s)}(Y|X) p^{(s)}_1(X,Z) ) \label{eq: 10003}\\
&+& \! (\!1 \!-\!\lambda\!) d  (p^{(u)}\!(Y|X\!)  p^{(u)}_2 \!(X,\!Z) || p^{(s)}\!(Y|X\!)  p^{(s)}_2 \!(X,\!Z)\!) \! \label{eq: 10004}\\
&\geq& \! d(p^{(u)}(Y|X) p^{(u)}_{\lambda}(X,Z) || p^{(s)}(Y|X) p^{(s)}_{\lambda}(X,Z))  \label{eq: tata28}
\end{eqnarray}where (\ref{eq: 10003}) and (\ref{eq: 10004}) {are due to} (\ref{eq: markov 1}) and (\ref{eq: markov 2}); (\ref{eq: tata28}) {is} due to (\ref{eq: tata29}), (\ref{eq: tata292}), and the convexity of $d(\cdot || \cdot)$.

Let $f_{\lambda}$ {be} the corresponding function that induces the joint distribution $p^{(u)}_{\lambda}(X,Z)$ and $p^{(s)}_{\lambda}(X,Z)$.  Define:
\begin{equation}
\label{eq: tata30}
    \Delta_{\lambda} = \int_{\bm{x} \in \mathcal{X}} \int_{\bm{z} \in \mathcal{Z}} p^{(s)}_{\lambda}(\bm{x},\bm{z}) \ell(g(\bm{z}), y^{(s)}(\bm{x})) \, d\bm{x} d\bm{z}.
\end{equation}
By definition of $T(\Delta)$ in Definition \ref{def: 3}, we have:
\begin{equation}
 \label{eq: tata42}
d(p^{(u)}\!(Y|X) \, p^{(u)}_{\lambda}\!(X,Z) || p^{(s)}(Y|X) \, p^{(s)}_{\lambda}(X,Z)) \!\geq\! T(\Delta_{\lambda}). 
\end{equation}Combine (\ref{eq: tata28}) and (\ref{eq: tata42}):
\begin{equation}
 \label{eq: tata43}
 \lambda T(\Delta_1) + (1-\lambda)T(\Delta_2) \geq  T(\Delta_{\lambda}).   
\end{equation}
That said, the left-hand side of (\ref{eq: prove convexity of T(V)}) is greater or equal to $T(\Delta_{\lambda})$. {Next,} we show that:
\begin{equation}
 \label{eq: tata44}
 T(\Delta_{\lambda}) \geq T(\lambda \Delta_1 + (1-\lambda) \Delta_2).
\end{equation}
Since $T(\Delta)$ is non-increasing, (\ref{eq: tata44}) is equivalent to:
 \begin{equation}
 \label{eq: tata45}
 \Delta_{\lambda}  \leq  \lambda \Delta_1 + (1-\lambda) \Delta_2.   
 \end{equation}
Indeed, we have:
 \begin{eqnarray}
 &\!& \! \Delta_{\lambda} = \int_{\bm{x}} \int_{\bm{z}} p^{(s)}_{\lambda}(\bm{x},\bm{z}) \ell(g(\bm{z}), y^{(s)}(\bm{x})) d\bm{x} d\bm{z} \label{eq: tata47}\\
 &=& \lambda \int_{\bm{x}} \int_{\bm{z}} p^{(u)}_1(\bm{x},\bm{z}) \ell(g(\bm{z}), y^{(s)}(\bm{x})) d\bm{x} d\bm{z} \label{eq: tata49-a}\\
 &+&  (1-\lambda) \int_{\bm{x}} \int_{\bm{z}} p^{(u)}_2(\bm{x},\bm{z}) \ell(g(\bm{z}),y^{(s)}(\bm{x})) d\bm{x} d\bm{z} \label{eq: tata49}\\
 &\leq& \lambda \Delta_1 + (1-\lambda) \Delta_2 \label{eq: tata50}
 \end{eqnarray}with (\ref{eq: tata47}) due to (\ref{eq: tata30}), (\ref{eq: tata49-a}) and (\ref{eq: tata49}) due to (\ref{eq: tata29}), (\ref{eq: tata50}) due to (\ref{eq: tata20}) and (\ref{eq: tata21}), respectively. From (\ref{eq: tata45}) and (\ref{eq: tata50}), (\ref{eq: tata44}) follows. Finally, from (\ref{eq: tata43}) and (\ref{eq: tata44}), (\ref{eq: prove convexity of T(V)}) follows. The proof is complete.
\end{proof}

It is worth noting that the convexity of $d(\cdot || \cdot)$ is not a restricted condition, indeed, most of the divergence functions, for example, the Kullback-Leibler (KL) divergence is convex. 

Theorem \ref{theorem: 1} shows that only enforcing a small distribution discrepancy between domains will increase the classification risk and vice-versa.

\begin{table*}[ht]
\caption{Classification accuracy of 12 tested algorithms on PACS, VLCS, and C-MNIST datasets using the  Training-domain validation method (Traditional) proposed in \cite{gulrajani2020search} \textit{vs.} using our new validation method.
}
\vspace{2 pt}
\renewcommand{\arraystretch}{1}
\resizebox{\textwidth}{!}{
\begin{tabular}{@{}cccccccccccccc@{}}
\toprule
\textbf{Algorithm} & \textbf{Fish} \cite{shi2021gradient} & \textbf{IRM} \cite{arjovsky2020invariant} & \textbf{GDRO} \cite{sagawa2019distributionally} & \textbf{Mixup} \cite{xu2020adversarial} & \textbf{CORAL} \cite{sun2016deep} & \textbf{MMD} \cite{li2018domain} & \textbf{DANN} \cite{ganin2016domain} & \textbf{CDANN} \cite{li2018deep} & \textbf{MTL} \cite{blanchard2021domain} & \textbf{VREx} \cite{krueger2021out} & \textbf{RSC} \cite{huang2020self} & \textbf{SagNet} \cite{nam2021reducing} & \textbf{Wins} \\ \midrule
\begin{tabular}[c]{@{}c@{}}PACS\\ (Traditional)\end{tabular} & \textbf{84.6} & 84.9 & 84.2 & 83.3 & \textbf{85.1} & 83.6 & 84.6 & \textbf{86.4} & 83.0 & \textbf{84.5} & \textbf{85.2} & 83.7 &  \\
\begin{tabular}[c]{@{}c@{}}PACS\\ (Ours)\end{tabular} & 82.0 & \textbf{85.3} & \textbf{84.3} & \textbf{85.3} & 84.9 & \textbf{85.0} & \textbf{84.9} & 82.0 & \textbf{84.2} & 84.2 & 81.3 & \textbf{85.1} & \multicolumn{1}{l}{7/12} \\ \midrule
\begin{tabular}[c]{@{}c@{}}VLCS\\ (Traditional)\end{tabular} & \textbf{79.4} & 76.0 & 78.1 & 77.4 & 76.8 & \textbf{78.5} & 77.8 & 79.2 & 77.3 & 76.4 & \textbf{78.6} & \textbf{80.5} &  \\
\begin{tabular}[c]{@{}c@{}}VLCS\\ (Ours)\end{tabular} & {77.5} & \textbf{79.2} & \textbf{79.6} & \textbf{77.6} & \textbf{78.8} & {78.0} & \textbf{78.5} &  \textbf{80.3} & \textbf{78.2} & \textbf{78.6} & 76.1 & 79.3 &\multicolumn{1}{l}{8/12} \\ \midrule
\begin{tabular}[c]{@{}c@{}}CMNIST \\ (Traditional)\end{tabular} & \textbf{10.0} & 10.0 & 10.2 & \textbf{10.4} & 9.7 & \textbf{10.4} & 10.0 & 9.9 & 10.5 & 10.2 & 10.2 & 10.4 &  \\
\begin{tabular}[c]{@{}c@{}}CMNIST \\ (Ours)\end{tabular} & 9.7 & \textbf{10.9} & \textbf{12.6} & 10.3 & \textbf{11.2} & 9.9 & \textbf{11.1} & \textbf{10.2} & \textbf{11.5} & \textbf{15.6} & \textbf{13.8} & \textbf{10.5} & \multicolumn{1}{l}{9/12} \\ \bottomrule
\end{tabular}}
\label{table:res}
\end{table*}

\section{A new validation method}
\label{sec: method}

Based on Theorem \ref{theorem: 1}, we argue that to {select} a good model for unseen domains, one must account for both the classification risk and the domain discrepancy not only in the training process but also in the validation process. Note that state-of-the-art model evaluation methods for DG are mainly based on the classification risk or, equivalently, the classification accuracy \cite{9847099} \cite{gulrajani2020search} on the validation set to select the models. Given this fact, we {propose} to select a model that minimizes the following objective function on the validation set:
\begin{equation}
\large
\label{eq: practical_loss}
    L_{\textup{Validation loss}} = \beta (1-\alpha)L_{\textup{Classification risk}} + \alpha L_{\textup{Domain-discrepancy loss}}
\end{equation}where $\alpha$ is the convex combination hyper-parameter and $\beta$ is the scale hyper-parameter that supports the combination of objectives with different scales. 

It is pretty clear that the cross-entropy loss is a good representation of classification risk. However, it is hard to {choose the measure for quantifying} the domain-discrepancy loss. Indeed, there exist various definitions of domain discrepancy. Several works characterize the domain-discrepancy via the difference in the marginal distributions \cite{li2018domain,sun2016deep}, other works measure {it} by the mismatch in conditional distributions \cite{arjovsky2020invariant}. We believe that finding a good measure for domain discrepancy is still an open problem. Therefore, in this short paper, we decide to use the widely accepted Maximum Mean Discrepancy (MMD) loss \cite{li2018domain} in the feature space to quantify the domain discrepancy. {We also} acknowledge that {though MMD measure is extensively used, it may not be the optimal choice.}

In practice, we found that MMD loss is at the same scale as the cross-entropy loss when the training process is stable, we thus choose $\beta$ as 1. For $\alpha$, we consider the classification performance {as the more important goal and thus} heuristically choose $\alpha$ as 0.2. From our experiments, we found that the performance of our validation method is robust to small values of $\alpha$ within the range of $[0.1, 0.3]$.
{One more insight from Theorem \ref{theorem: 1} is that it is advisable to avoid extreme points in $\Delta$ (classification error) to maintain a balance between the model's generalization and prediction capabilities. This means the classification error should not be too small or too large.} Thus, for each hyper-parameter configuration, we sort the validation cross-entropy loss in ascending order and only pick the models that produce 5\% to 50\% percentile of the validation cross-entropy loss as a subset of candidates for model selection.
Our implementation is released at \href{https://github.com/thuan2412/A-principled-approach-for-model-validation-for-domain-generalization}{this link}\footnote{https://github.com/thuan2412/A-principled-approach-for-model-validation-for-domain-generalization}.

\section{Numerical results}
\label{sec: numerical results}

We compare the proposed model selection method with the Training-domain validation method described in \cite{gulrajani2020search} on {three} datasets: PACS, {VLCS, and} Colored-MNIST (C-MNIST)  using  DomainBed package and 12 different DG algorithms provided there \cite{gulrajani2020search}. Recall that the Training-domain validation method chooses the model that produces the highest validation accuracy, while our method selects the model that minimizes the objective function in (\ref{eq: practical_loss}). For PACS {and VLCS datasets}, we report the average test accuracy over 4 different tasks with each time leaving one domain out as the unseen domain. For the C-MNIST dataset, we only focus on the most difficult domain, where the correlation between the label and the color of the unseen domain is completely different from the seen domains and no algorithm can achieve more than 10.5\% points accuracy \cite{gulrajani2020search}. 

The validation set is formed using 20\% data from each seen domain, denoted as the training-domain validation set in \cite{gulrajani2020search}. We follow exactly the same settings and training routine used in DomainBed and conduct 20 trials of random search over a joint distribution of hyper-parameters for each task per algorithm. For the MMD loss implementation, we directly use the code provided in DomainBed package.
We train each model for 5000 steps. The validation cross-entropy loss, MMD loss, and validation accuracy are recorded every 100 steps for VLCS dataset and every 300 steps for PACS and C-MNIST datasets.

{With $\alpha=0.2, \beta=1$,} the performance of each algorithm under different validation methods on PACS, VLCS and Colored-MNIST datasets is shown in Table \ref{table:res}. We refer to the Training-domain validation method as ``Traditional" and the proposed method as ``Ours". For the PACS dataset, the proposed validation method can select slightly better models for seven out of twelve DG algorithms. For the remaining five DG algorithms, our method achieves comparable performance with the ``Traditional" method on CORAL \cite{sun2016deep} and VREx \cite{krueger2021out}. However, for Fish \cite{shi2021gradient}, CDANN \cite{li2018deep} and RSC\cite{huang2020self}, we observe a performance deterioration. 
{The effectiveness of the proposed method can be more easily observed on VLCS dataset, where eight out of twelve DG algorithms get an improved model selected, with the improvement varies from 0.2\% to 3.2\%.}
For the C-MNIST dataset, the proposed validation method consistently selects models with better performance compared with the ``Traditional" validation method. 
{Accuracy improves for nine out of twelve tested algorithms with the most significant improvement for VREx \cite{krueger2021out} method by 5.4\%.} 

\section{Conclusion}
\label{sec: conclusion}
{By showing the trade-off between minimizing the classification risk and domain discrepancy, we demonstrate that the traditional model selection methods may not be suitable for DG problem and propose a new model selection method that considers both objectives. While our approach outperforms traditional methods on several DG algorithms and datasets, it lacks an automatic hyper-parameter tuning strategy. Note that the domain discrepancies may vary across different datasets, one may not expect the same optimal values of $\alpha$ and $\beta$ for all datasets. Determining the ``optimal" ones could be a hard problem both practically and theoretically. We thus leave it as an open problem for future work. Despite this limitation, we believe our approach provides insight and initial results for exploring new model selection methods specific for DG problem.
}

\bibliography{example_paper}

\begin{thebibliography}{10}

\bibitem{albuquerque2019generalizing}
I.~Albuquerque, J.~Monteiro, M.~Darvishi, T.~H. Falk, and I.~Mitliagkas.
\newblock Generalizing to unseen domains via distribution matching.
\newblock {\em arXiv preprint arXiv:1911.00804}, 2019.

\bibitem{arjovsky2020invariant}
M.~Arjovsky, L.~Bottou, I.~Gulrajani, and D.~Lopez-Paz.
\newblock Invariant risk minimization.
\newblock {\em stat}, 1050:27, 2020.

\bibitem{arpit2021ensemble}
D.~Arpit, H.~Wang, Y.~Zhou, and C.~Xiong.
\newblock Ensemble of averages: Improving model selection and boosting
  performance in domain generalization.
\newblock {\em Advances in Neural Information Processing Systems}, 2021.

\bibitem{ben2006analysis}
S.~Ben-David, J.~Blitzer, K.~Crammer, and F.~Pereira.
\newblock Analysis of representations for domain adaptation.
\newblock {\em Advances in neural information processing systems}, 19, 2006.

\bibitem{blanchard2021domain}
G.~Blanchard, A.~A. Deshmukh, {\"U}.~Dogan, G.~Lee, and C.~Scott.
\newblock Domain generalization by marginal transfer learning.
\newblock {\em The Journal of Machine Learning Research}, 22(1):46--100, 2021.

\bibitem{cover1999elements}
T.~M. Cover.
\newblock {\em Elements of information theory}.
\newblock John Wiley \& Sons, 1999.

\bibitem{david2010impossibility}
S.~B. David, T.~Lu, T.~Luu, and D.~P{\'a}l.
\newblock Impossibility theorems for domain adaptation.
\newblock In {\em Proceedings of the Thirteenth International Conference on
  Artificial Intelligence and Statistics}, pages 129--136. JMLR Workshop and
  Conference Proceedings, 2010.

\bibitem{ganin2016domain}
Y.~Ganin, E.~Ustinova, H.~Ajakan, P.~Germain, H.~Larochelle, F.~Laviolette,
  M.~Marchand, and V.~Lempitsky.
\newblock Domain-adversarial training of neural networks.
\newblock {\em The journal of machine learning research}, 17(1):2096--2030,
  2016.

\bibitem{gulrajani2020search}
I.~Gulrajani and D.~Lopez-Paz.
\newblock In search of lost domain generalization.
\newblock In {\em International Conference on Learning Representations}, 2020.

\bibitem{huang2020self}
Z.~Huang, H.~Wang, E.~P. Xing, and D.~Huang.
\newblock Self-challenging improves cross-domain generalization.
\newblock In {\em European Conference on Computer Vision}, pages 124--140.
  Springer, 2020.

\bibitem{krueger2021out}
D.~Krueger, E.~Caballero, J.-H. Jacobsen, A.~Zhang, J.~Binas, D.~Zhang,
  R.~Le~Priol, and A.~Courville.
\newblock Out-of-distribution generalization via risk extrapolation (rex).
\newblock In {\em International Conference on Machine Learning}, pages
  5815--5826. PMLR, 2021.

\bibitem{li2021invariant}
B.~Li, Y.~Shen, Y.~Wang, W.~Zhu, D.~Li, K.~Keutzer, and H.~Zhao.
\newblock Invariant information bottleneck for domain generalization.
\newblock In {\em Proceedings of the AAAI Conference on Artificial
  Intelligence}, volume~36, pages 7399--7407, 2022.

\bibitem{li2018domain}
H.~Li, S.~J. Pan, S.~Wang, and A.~C. Kot.
\newblock Domain generalization with adversarial feature learning.
\newblock In {\em Proceedings of the IEEE Conference on Computer Vision and
  Pattern Recognition}, pages 5400--5409, 2018.

\bibitem{li2018deep}
Y.~Li, X.~Tian, M.~Gong, Y.~Liu, T.~Liu, K.~Zhang, and D.~Tao.
\newblock Deep domain generalization via conditional invariant adversarial
  networks.
\newblock In {\em Proceedings of the European Conference on Computer Vision
  (ECCV)}, pages 624--639, 2018.

\bibitem{lyu2021barycentric}
B.~Lyu, T.~Nguyen, P.~Ishwar, M.~Scheutz, and S.~Aeron.
\newblock Barycentric-alignment and reconstruction loss minimization for domain
  generalization.
\newblock {\em arXiv preprint arXiv:2109.01902}, 2021.

\bibitem{nam2021reducing}
H.~Nam, H.~Lee, J.~Park, W.~Yoon, and D.~Yoo.
\newblock Reducing domain gap by reducing style bias.
\newblock In {\em Proceedings of the IEEE/CVF Conference on Computer Vision and
  Pattern Recognition}, pages 8690--8699, 2021.

\bibitem{nguyen2022conditional}
T.~Nguyen, B.~Lyu, P.~Ishwar, M.~Scheutz, and S.~Aeron.
\newblock Conditional entropy minimization principle for learning domain
  invariant representation features.
\newblock In {\em 2022 26th International Conference on Pattern Recognition
  (ICPR)}, pages 3000--3006. IEEE, 2022.

\bibitem{nguyen2022joint}
T.~Nguyen, B.~Lyu, P.~Ishwar, M.~Scheutz, and S.~Aeron.
\newblock Joint covariate-alignment and concept-alignment: a framework for
  domain generalization.
\newblock In {\em 2022 IEEE 32nd International Workshop on Machine Learning for
  Signal Processing (MLSP)}, pages 1--6. IEEE, 2022.

\bibitem{sagawa2019distributionally}
S.~Sagawa*, P.~W. Koh*, T.~B. Hashimoto, and P.~Liang.
\newblock Distributionally robust neural networks.
\newblock In {\em International Conference on Learning Representations}, 2020.

\bibitem{shi2021gradient}
Y.~Shi, J.~Seely, P.~Torr, S.~N, A.~Hannun, N.~Usunier, and G.~Synnaeve.
\newblock Gradient matching for domain generalization.
\newblock In {\em International Conference on Learning Representations}, 2022.

\bibitem{sun2016deep}
B.~Sun and K.~Saenko.
\newblock Deep coral: Correlation alignment for deep domain adaptation.
\newblock In {\em European conference on computer vision}, pages 443--450.
  Springer, 2016.

\bibitem{wald2021calibration}
Y.~Wald, A.~Feder, D.~Greenfeld, and U.~Shalit.
\newblock On calibration and out-of-domain generalization.
\newblock {\em Advances in neural information processing systems},
  34:2215--2227, 2021.

\bibitem{xu2020adversarial}
M.~Xu, J.~Zhang, B.~Ni, T.~Li, C.~Wang, Q.~Tian, and W.~Zhang.
\newblock Adversarial domain adaptation with domain mixup.
\newblock In {\em Proceedings of the AAAI Conference on Artificial
  Intelligence}, volume~34, pages 6502--6509, 2020.

\bibitem{ye2021towards}
H.~Ye, C.~Xie, T.~Cai, R.~Li, Z.~Li, and L.~Wang.
\newblock Towards a theoretical framework of out-of-distribution
  generalization.
\newblock {\em Advances in Neural Information Processing Systems},
  34:23519--23531, 2021.

\bibitem{zhao2019learning}
H.~Zhao, R.~T. Des~Combes, K.~Zhang, and G.~Gordon.
\newblock On learning invariant representations for domain adaptation.
\newblock In {\em International Conference on Machine Learning}, pages
  7523--7532. PMLR, 2019.

\bibitem{9847099}
K.~Zhou, Z.~Liu, Y.~Qiao, T.~Xiang, and C.~C. Loy.
\newblock Domain generalization: A survey.
\newblock {\em IEEE Transactions on Pattern Analysis and Machine Intelligence},
  pages 1--20, 2022.

\end{thebibliography}

\end{document}